\documentclass[a4paper,11pt]{article}
\usepackage[round]{natbib}
\renewcommand{\cite}{\citep}
\usepackage{authblk}
\usepackage{amsthm}

\newtheorem{theorem}{Theorem}[section]

\newtheorem{proposition}{Proposition}[section]
\newtheorem{corollary}{Corollary}[proposition]
\theoremstyle{definition}
\newtheorem{Definition}{Definition}[section]

\usepackage{times}
\usepackage{microtype}
\usepackage{graphicx}
\usepackage{subfig}
\usepackage{booktabs}
\usepackage{multirow}
\usepackage{nicefrac}
\usepackage{textcomp}
\usepackage{comment}
\usepackage{amsmath,amssymb,amsfonts}
\usepackage[utf8]{inputenc}
\usepackage{wrapfig}
\def\BibTeX{{\rm B\kern-.05em{\sc i\kern-.025em b}\kern-.08em
    T\kern-.1667em\lower.7ex\hbox{E}\kern-.125emX}}
  
\usepackage{enumitem} % Ohsaka
\setlist{itemsep=0pt}

\usepackage{xcolor}
\definecolor{mygreen}{rgb}{0.17,0.63,0.17}
\definecolor{mydarkred}{rgb}{0.6,0,0}
\definecolor{mydarkgreen}{rgb}{0,0.6,0}

\usepackage{mathdef} % homemade

\usepackage{url}
\usepackage[colorlinks,
 	linkcolor=mydarkred,
 	citecolor=mydarkgreen]{hyperref}

\begin{document}
\sloppy

\title{Predictive Optimization \\with Zero-Shot Domain Adaptation}
\author[1,2]{Tomoya Sakai}
\author[1]{Naoto Ohsaka}
\affil[1]{NEC Corporation}
\affil[2]{RIKEN}
\affil[ ]{\footnotesize \texttt{ \{tomoya\_sakai, ohsaka\}{\fontfamily{ptm}\selectfont @}nec.com }} 
\date{}

\maketitle

\begin{abstract}
Prediction in a new domain without any training sample, called \emph{zero-shot domain adaptation} (ZSDA), is an important task in domain adaptation.
While prediction in a new domain has gained much attention in recent years, in this paper, we investigate another potential of ZSDA.
Specifically, instead of predicting responses in a new domain, we find a description of a new domain given a prediction.
The task is regarded as \emph{predictive optimization}, but existing predictive optimization methods have not been extended to handling multiple domains.
We propose a simple framework for predictive optimization with ZSDA and analyze the condition in which the optimization problem becomes convex optimization.
We also discuss how to handle the interaction of characteristics of a domain in predictive optimization.
Through numerical experiments, we demonstrate the potential usefulness of our proposed framework.
\end{abstract}

\providecommand{\keywords}[1]
{
  \small	
  \textbf{\textit{Keywords---}} #1
}

\keywords{Predictive optimization, Zero-shot domain adaptation, Convex optimization}

\section{Introduction}
Prediction in a new domain without any training samples, called \emph{zero-shot domain adaptation} (ZSDA) \cite{Diff-CV:Yang+Hospedales:2015,ICLR:Yang+Hospedales:2015}, is an important task in domain adaptation.
To this end, an approach to utilize domain descriptions \cite{Diff-CV:Yang+Hospedales:2015,ICLR:Yang+Hospedales:2015}, called \emph{domain attributes}, has been developed.
A goal of ZSDA is to obtain predictions in an \emph{unseen} domain in which we did not observe any training samples.
An application of ZSDA is the sales prediction of new products; regarding domains as products and given product attributes and sales data, we can use ZSDA to the sales prediction of a customer for a new product.
Thanks to ZSDA, we can predict the response of input in an unseen domain; however, one potential aspect of ZSDA has been overlooked.

We demonstrate another potential of ZSDA; by reversing the ZSDA prediction process, we can optimize domain attributes so that an evaluation metric of responses over customers is maximized, referred to as \emph{attribute optimization} as shown in Figure~\ref{fig:illust_po_with_zsda}.
%that is likely to achieve high responses, 
That is, instead of predicting responses given new domain attributes as in ZSDA, our task is to find new domain attributes given a prediction.
In our example of new product prediction, we optimize an average of new product sales with respect to product attributes over a pre-specified customer group.
The obtained product attributes would be useful for designing a new product.

Our attribute optimization task can be regarded as \emph{predictive optimization} \cite{KDD:Ito+Fujimaki:2017,NeurIPS:Donti+etal:2017,ICML:Ito+etal:2018}, in which the goal is to optimize predicted outputs in terms of input variables.
There are various applications of predictive optimization: water distribution management \cite{ICDMWS:Vzliobite+Pechenizkiy:2010}, retail price optimization \cite{NeurIPS:Ito+Fujimaki:2016,KDD:Ito+Fujimaki:2017}, and grid scheduling \cite{NeurIPS:Donti+etal:2017}.
%and inventory optimization [after acceptances of this paper]
%Optimization based on prediction has been studied as predictive optimization
%and gathered active attentions recently [cite].
However, existing studies of predictive optimization mainly focus on a single domain, and the case of multiple domains has yet to be considered.
While we can use existing methods for each domain independently, it would not exploit the structures and similarities across multiple domains.
Moreover, it is not straightforward to optimize domain attributes for finding, e.g., a promising product in existing methods.
%Thus, the case of multiple domains has not been considered and it is unknown 
%if the predictive optimization problems taking into account multiple domains 
%can be solved efficiently.
%[[we cannot solve such problems with existing PO... (examples)]]
%[[Simply applying existing methods to this task would not work for the reason that %brahbrahbrah,]]

\begin{figure}[t]
    \centering
    \includegraphics[clip,width=.55\columnwidth]{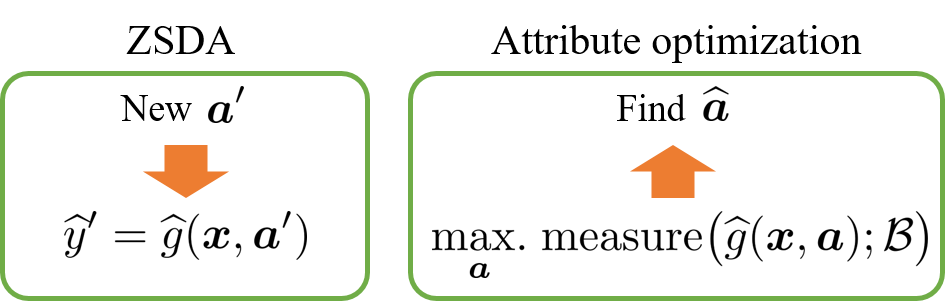}
    \caption{Illustration of attribute optimization.
    See Section~\ref{sec:prob_setting} for notations.
    In ZSDA, the purpose is to obtain a prediction in a new unseen domain $\ba'$.
    In contrast, our goal of attribute optimization is to find $\ba$ such that a certain ``measure'' over $\cB$ is maximized.
    %``$\underset{\bx\in\cB}{\textrm{measure}}$'' in the figure) is maximized.
    }\label{fig:illust_po_with_zsda}
\end{figure}

In this paper, we propose a novel simple framework for attribute optimization.
Given domain attributes, inputs, and responses, our framework first trains a prediction function and then optimizes a measure computed by predicted outputs to find a new domain attribute.
Our method can handle continuous and discrete domain attributes and concave measure functions as objective functions.
Moreover, it can deal with cases when two or more domain-attribute variables are dependent, i.e., the interaction of domain attributes.
%Given [[INPUT]], our framework does brah-brah-brah to generate [[OUTPUT]].
%[[What can proposed method do]]
%[[Some kind of limitation... (we can handle with discrete, continuous, objective function' shape)]]
%[[linear constraint -> interaction]]
%[[the situation in which two or more attribute variables are dependent]]
By regarding domains as \emph{objects}, domain attributes as \emph{designs}, and objective functions as \emph{sales, quality, or durability}, we use attribute optimization to discover the design of new best-selling, high-quality, and durable objects.
Our attribute optimization framework can be applied to the following concrete tasks:
\begin{itemize}[leftmargin=*]
    \item \textbf{New product design}:
    Consider that we are given sales, consumer features, and product characteristics data to design a new product by combining best-selling products' ingredients.
    Our framework easily satisfies this demand by regarding domains as products, domain attributes as product characteristics, features as consumer features, and the objective function as sales.
    
    \item \textbf{New tourist spot design}:
    Suppose that we are given reputations of tourist spots, tourist features, and characteristics of tourist spots data to design a new tourist spot on the basis of characteristics of tourist spots.
    Our framework easily fits this task by regarding a domain as a tourist spot, domain attributes as characteristics of a tourist spot, features as tourist features, and the objective function is the average reputation.
\end{itemize}

Our technical contributions for achieving the above-mentioned attribute optimization framework are as follows:
\begin{itemize}[leftmargin=*]
    \item We propose a simple framework for \emph{attribute optimization}.
    The main components of our framework are training predictors (ZSDA) and subsequent mathematical optimization (Sections~\ref{sec:prediction} and \ref{sec:optimization}).
    \item We devise conditions on objective functions and constraints whose corresponding optimization problem can be cast to \emph{convex optimization}, which is quickly solvable by off-the-shelf solvers (Section~\ref{sec:impl}).
    \item We provide practical examples of objective functions and constraints that meet the devised conditions (Section~\ref{sec:examples}).
    \item We describe how to handle interactions between domain-attribute variables.
%i.e., the situation in which two or more attribute variables are dependent,
    We show that the second-order interaction of $0$-$1$ domain attributes can be relaxed to \emph{semidefinite programming} (Section~\ref{sec:sdp}).
    \item We establish theoretical analyses of the proposed framework, which bounds the generalization error of the prediction method and approximation factors of the optimization method (Section~\ref{sec:theory}).
    \item We demonstrate the potential usefulness of our proposed method on synthetic and benchmark datasets (Section~\ref{sec:experiments}).
\end{itemize}

\section{Related work}

\subsection{Predictive optimization}
Predictive optimization has been applied to several real-world applications such as water distribution management \citep{ICDMWS:Vzliobite+Pechenizkiy:2010}, retail price optimization \cite{NeurIPS:Ito+Fujimaki:2016,KDD:Ito+Fujimaki:2017}, grid scheduling \cite{NeurIPS:Donti+etal:2017}, and inventory optimization \citep{SDM:Ohsaka+etal:2020}.
In existing work, a set of input-output samples is collected from a target domain; we train a prediction function and then optimize \emph{features of input} to maximize a certain measure of output in the target domain.
For example, in price optimization, we have sales for each item at a certain price. 
We first train a model to output the sales of an item from a price.
We then optimize the prices of items such that the total sales is maximized.
In existing work, the way of using item information, i.e., domain attributes, is not considered.
In addition, prediction and optimization including features of input and domain attributes are not trivial.

In contrast, our approach optimizes \emph{domain attributes} to maximize a certain measure of output for a set of fixed input samples.
In other words, we can find a prospective target domain for specific input samples through the optimized domain attributes.
Moreover, existing studies often consider a single domain while our study considers multiple domains.
To the best of our knowledge, this is the first study of predictive optimization for multiple domains with its attribute information.

\subsection{Data-driven design}
An application of our method is \emph{data-driven design} for new products.
For this purpose, several methods based on machine learning have been proposed recently \citep{KDD:Koutra+etal:2017,ICDM:Kang+etal:2017,RecSys:Vo+etal:2018}.
Among them, \citet{KDD:Koutra+etal:2017} is related to our problem setting, which considers multiple domains.
They also considered optimization of a domain for fixed target users and applied their method to movie design for target users by regarding a domain as a movie.
The method learns user-preferences through a tripartite graph of users, movies, and movie attributes.
The movie attributes are optimized by a greedy approach, which is optimal under specific assumptions~\citep{KDD:Koutra+etal:2017}.

Compared with the work of \citep{KDD:Koutra+etal:2017}, our method is not limited to data having graph structures and enables us to use various prediction models and optimization algorithms, as shown in Section~\ref{sec:proposed}.
In addition, we demonstrate the effectiveness of our method on various real-world datasets in Section~\ref{sec:experiments}.

\begin{comment}
\begin{table}[t]
\centering
\caption{
Caption
}
\begin{tabular}{lll}
\toprule
{} & Samples & What is optimized \\
\midrule
Existing method & input-output samples & features of input \\
\midrule
\multirow{2}{*}{Our method} & input-output samples & \multirow{2}{*}{domain-attributes} \\
 & domain-attributes & \\
\bottomrule 
\end{tabular}
\end{table}
\end{comment}

%\section{Problem settings}
\section{Background}
\label{sec:background}

\subsection{Problem settings}
\label{sec:prob_setting}
%We first define the attribute optimization problem and explain
%the potential applications of attribute optimization.

%\paragraph{Problem settings}
Let $\bx_{}^{(t)}\in\Re_{}^{d}$ be a feature vector and $y_{}^{(t)}\in\Re$ be a response in a domain $t\in\{1,\ldots,T\}$, where $d$ is a positive integer and $T$ denotes the number of domains.
We have a set of observations $\cD_{}^{(t)}:=\{(\bx_{i}^{(t)},y_{i}^{(t)})\}_{i=1}^{n_{}^{(t)}}$ for each $t$, where $n_{}^{(t)}$ is the number of samples in $t$.
For each domain, we assume that a description of a domain is available and they can be expressed as $m$ \emph{domain-attribute} variables.
We denote a domain-attribute vector for $t$ by $\ba_{}^{(t)}\in\Re^m$ and a set of domain-attribute vectors by $\cA:=\{\ba_{}^{(t)}\}_{t=1}^{T}$.
We define a dataset as $\cD:=\{ (\cD_{}^{(t)}, \ba_{}^{(t)}) \}_{t=1}^{T}$.

Let $g\colon\Re^d\times\Re^m\to\Re$ be a prediction function and $\cB$ be a set of target feature vectors.
% Let $F$ be an aggregate functional such that $f(\ba)=F(g(\,\cdot\, , \ba); \cB)$
%and $f$ be an objective function that returns a gain 
%of an object, described by $\ba$.
Let $f$ be an objective function that returns a gain of a domain described by $\ba$.
That is, we measure the value in a domain through $f(\ba)$.
An example of $f$ is a mean response: $f(\ba)=(1/|\cB|)\sum_{\bx\in\cB}^{}g(\bx,\ba)$, where $|\cdot|$ denotes the size of a set.
%For the sake of brevity, we sometimes omit the notation $\cB$ 
%and use $\gain(g(\cdot, \ba))$ if it is obvious from the %context.

The goal of the \emph{attribute optimization} problem is to find a new domain-attribute vector optimized for a gain function.

\subsection{Feature-unaware approach (FUA)}
An approach to optimize domain attributes is that we first learn a function from $\ba^{(t)}$ to, e.g., the mean of $\{ y_i^{(t)} \}_{i=1}$, and then find $\ba$ which maximizes the learned function.
We call the above approach the \emph{feature-unaware approach} (FUA).
Specifically, let $\overline{y}^{(t)}$ be the average of responses of domain $t$.
As a function, we use a linear model defined as $s(\ba)=\bw^\top\ba$, where $\bw\in\Re^m$ is the parameter vector.
We then train $s$ with $\{ (\ba^{(t)}, \overline{y}^{(t)} \}_{t=1}^T$ by, say, regularized least squares, i.e., ridge regression.
Let $\bwh$ be the estimated parameter obtained after training.
To optimize domain attributes, an approach is to select a domain-attribute variable whose corresponding weight of $\bwh$ to satisfy the user-defined and system-derived constraints.
Another approach is to formulate an optimization problem and solve it.

The FUA is simple, but it ignores dependence on features $\bx$.
We thus cannot optimize domain-attributes for each $\bx$ or a group of $\bx$, and cannot use the measure taking a distribution of features into account, which will be introduced in Section~\ref{sec:gain_func}.
Moreover, since the number of training samples is $T$, the use of complex models, such as neural networks, for $s$ leads to ill-posed problems, resulting in that an inaccurate estimation of response.
In contrast, our proposed approach takes features of input into account.
Moreover, we can use complex models for estimating response since the number of training samples is much larger than the FUA. 

\section{Proposed method}
\label{sec:proposed}
Our method consists of two steps: i) a prediction step that estimates a prediction function from $\cD$, and ii) an optimization step that solves an optimization problem under a user-preferred gain function.

\subsection{Prediction step}
\label{sec:prediction}
In the prediction step, we train a parameterized prediction function, $g$, with a training dataset $\cD$.
A simple example of $g$ is $g(\bx,\ba)=\bx^\top\bTheta\ba$, where $\bTheta\in\Re_{}^{d\times m}$ are the parameters to be learned.
For a fixed domain-attribute vector $\ba_{}^{(t)}$, we can regard $g$ as a prediction function for a specific domain $t$, e.g., $g_{}^{(t)}\colon\Re_{}^{d}\to\Re$.

Let $\ell\colon\Re\times\Re\to\Re_{\geq0}^{}$ be a loss function.
By solving the following optimization problem, we obtain a learned prediction function as
\begin{align*}
\gh:=\argmin_{g}\;%\btheta}\;
    \frac{1}{T}\sum_{t=1}^{T}\frac{1}{n_{}^{(t)}}\sum_{i=1}^{n}
    \ell\big(g(\bx_{i}^{(t)},\ba_{}^{(t)}),y_{i}^{(t)}\big) 
    + \lambda W(g),
\end{align*}
where $W$ is the regularization functional and $\lambda\geq0$ is the regularization parameter.

With the learned prediction function, we can estimate a response of $\bx$ even in a domain never seen before.
Let $\ba_{}^{\prime}$ be a domain-attribute vector for an unseen domain.
We can obtain a prediction in a domain that did not appear in a training dataset as $\gh(\bx,\ba_{}^{\prime})$.

Our idea is to reverse the aforementioned prediction process for finding a new domain-attribute vector that is likely to get a high gain as shown in Figure~\ref{fig:illust_po_with_zsda}.
That is, instead of obtaining a prediction given $\ba$, we find $\ba$ such that a prediction-based gain function is maximized.

\subsection{Optimization step}
\label{sec:optimization}
After we obtained a learned prediction function $\gh$, we move on to the optimization step.
As $\cB$ is a set of target feature vectors, we can regard it as a set of target users.
On the basis of $\gh$, we compute an estimate of an objective function $\fh$.

For a user-preferred gain function, we can find a promising domain-attribute vector by solving the following optimization problem:
\begin{align}
\begin{aligned}
&\maximize_{\ba\in\Re_{}^{m}} && \fh(\ba)  \\
& \subjectto && b_{j}^{}(\ba)\leq 0,\; j=1,\ldots,s \\
& && c_{j}^{}(\ba)=0,\; j=1,\ldots,t ,
\end{aligned}
\label{eq:opt_prob}
\end{align}
where $b_{j}^{}$ and $c_{j}^{}$ are an inequality and equality constraint function, respectively.
An example of a constraint is a budget constraint; if $\ba$ is the $0$-$1$ vector and $\bgamma\in\Re_{\geq0}^{m}$ is defined as a cost vector, a budget constraint can be expressed as $\ba_{}^{\top}\bgamma\leq C$, where $C$ is the user-specified constant.

By solving the optimization problem in Eq.~\eqref{eq:opt_prob}, we obtain a domain-attribute vector that is potentially new and will achieve a high gain.
In the subsequent section, we explain the conditions in which the optimization problem can be solved efficiently.

%\subsection{Efficient implementation}
%\subsection{Analyses}
\subsection{Linear-in-attribute model (LAM)}
\label{sec:impl}
In this section, we reveal conditions in which an optimization problem is regarded as a convex optimization problem.

%\paragraph{Linear-in-attribute model:}
We first define a model of a prediction function:
\begin{Definition}[Linear-in-attribute model]
A linear-in-attribute model (LAM) is defined as
%$g(\bx,\ba)=\ba_{}^{\top}\bphi(\bx)$,
\begin{align}
g(\bx,\ba)=\ba_{}^{\top}\bphi(\bx),
\label{eq:LAM}
\end{align}
where $\bphi(\bx):=(\phi_{1}^{}(\bx),\ldots,\phi_{m}^{}(\bx))_{}^{\top}\in\Re_{}^{m}$ is a basis function vector whose parameters are to be learned with training data and $_{}^{\top}$ denotes the transpose of a vector or matrix.
\end{Definition}
We denote the learned basis function vector by $\bphih$; an estimated response function is expressed as $\gh(\bx,\ba)=\ba_{}^{\top}\bphih(\bx)$.

Let $h\colon\Re_{}^{d}\to\Re$ be a prediction function. 
We then define a functional computing gain given $h$.
\begin{Definition}[Aggregate functional]
An aggregate functional $F$ computes gain from $h$ and $\cB$.
That is, the computed gain is given by $F(h;\cB)$.
For the sake of brevity, we omit the notation $\cB$ and use $F(h)$.
\end{Definition}
Let us denote $g(\,\cdot\,,\ba)$ by $g_{\ba}^{}\colon\Re_{}^{d}\to\Re$.
A gain function is expressed as $f(\ba)=F(g_{\ba}^{})$.
We then have the following proposition that characterizes a gain function:
\begin{proposition}
\label{thm:ccv}
If $F$ is a concave function and $g$ is a LAM, the gain function $f$ is concave.
\end{proposition}
\begin{proof}
Without loss of generality, we assume $d=1$.
The second derivative of $f$ with respect to $\ba$ is given by
$\partial_{\ba}^{2}f = \partial_{g_{\ba}^{}}^{2}F\cdot(\partial_{\ba}^{} g)_{}{^2}
    + \partial_{g_{\ba}^{}}^{}F\cdot\partial_{\ba}^{2} g $,
%\begin{align*}
%\partial\x{\ba}{2}f = \partial\x{g\x{\ba}{}}{2}F\cdot(\partial\x{\ba}{} g)\x{}{2}
%    + \partial\x{g\x{\ba}{}}{}F\cdot\partial\x{\ba}{2} g ,
%\end{align*}
where $\partial_{\ba}^{}=\frac{\partial}{\partial \ba}$ and $\partial_{g_{\ba}^{}}^{}=\frac{\partial}{\partial g_{\ba}^{}}$.
The second term becomes zero because $g$ is linear in attributes, i.e., $\partial_{\ba}^{2}g=0$.
Since $F$ is concave, $\partial_{g_{\ba}^{}}^{2}F\leq0$.
The first term is thus non-positive.
Then, $\partial_{\ba}^{2}f\leq0$, concluding that $f$ is concave.
\end{proof}

Proposition~\ref{thm:ccv} leads to the following corollary:
\begin{corollary}
If a LAM, $g$, and convex constraints, $b_{j}^{}$ and $c_{j}^{}$, are used and $F$ is concave, the optimization problem in Eq.~\eqref{eq:opt_prob} is a convex optimization problem.
\end{corollary}
For convex optimization problems, we can use efficient off-the-shelf solvers to obtain a solution.
In Section~\ref{sec:gain_func} and~\ref{sec:constraints}, we introduce useful candidates of $F$, $b_{j}^{}$, and $c_{j}^{}$.
Note that for higher accuracy, one can use non-convex models, rather than LAM, and use Bayesian optimization~\cite{TGO:Mockus+etal:1978}, but we do not pursue that direction because non-convex optimization is often time-consuming than convex optimization.

\section{Examples of objective functions and constraints}
\label{sec:examples}

\subsection{Aggregate functionals}
\label{sec:gain_func}
We introduce concave aggregate functionals.
% As mentioned, a linear-in-attribute model $g$ is assumed.
Recall that $h\colon\Re_{}^{d}\to\Re$ is a prediction function.

\paragraph{Mean response:}
A standard choice of $F$ is a mean response defined as 
%$F(h) = (1/|\cB|)\sum_{\bx\in\cB}^{}h(\bx)$.
\begin{align*}
F(h) = \frac{1}{|\cB|}\sum_{\bx\in\cB} h(\bx) .
\end{align*}
With $g$, the mean response with respect to $\ba$ is expressed as $f(\ba)=(1/|\cB|)\sum_{\bx\in\cB}^{}g(\bx, \ba)$.
For simplicity, we refer to the mean response aggregate functional as the mean gain function.
%We refer $g$ with the mean response functional to as the mean gain function.
A mean response can be interpreted easily; in our example of product sales prediction, maximization of a mean response over all users with respect to $\ba$ corresponds to finding a product that is likely to be preferred by all users on average.

On the implementation side, a mean response is linear, resulting in $f$ being concave by Proposition~\ref{thm:ccv}.
It thus enables us to obtain a solution efficiently under convex constraints.

\paragraph{Conditional value at risk (CVaR):}
A mean response is simple and a standard choice but it does not take into account a distribution of responses.
In practical applications, we are sometimes interested in a tail of a distribution, in particular, a group of customers whose responses are relatively lower than others.
If we maximize an objective that can capture a left tail of a response distribution, it corresponds to avoiding the situation in which users will put a lower rating on an object.
%maximizing entire responses and under a certain probability,

To treat a left tail of a response distribution, we can use \emph{conditional value at risk} (CVaR) \cite{JOR:Rockafellar+Uryasev:2000,JBF:Rockafellar+Uryasev:2002} at a significance level $0<\beta<1$, defined as 
\begin{align*}
\CVaR_{\beta}^{}(h):=\maximize_{\tau\in\Re}
\Big(\tau - \frac{1}{\beta |\cB|}\sum_{\bx\in\cB}^{}
    \max(0,\,\tau - h(\bx))\Big) .
%\label{eq:emp_cvar}
%\CVaR\x{\beta}{}(g\x{\ba}{}):=\maximize_{\tau\in\Re}
%\Big(\tau - \frac{1}{\beta |\cB|}\sum\x{\bx\in\cB}{}
%    \max(0,\,\tau - g\x{\ba}{}(\bx))\Big) .
\end{align*}
%We refer $g$ with the CVaR-based aggregate functional
%to as the CVaR-based gain function.
For brevity, we refer to the CVaR-based aggregate functional as the CVaR-based gain function.

A useful property of CVaR is concavity; the CVaR-based gain function with the LAM becomes concave from Proposition~\ref{thm:ccv}.
%The details of CVaR are given in Appendix~\ref{sec:details_of_cvar}.
%while a mean responses ignore such users who give lower scores to an object.

\subsection{Constraints}
\label{sec:constraints}
In this section, we introduce constraints that can be used in practical applications.

\paragraph{Continuous domain attributes:}
If a domain-attribute variable is a continuous value and a mean response is used as an objective function, we can maximize the objective function as much as we can by increasing the magnitude of the domain-attribute variable; it is, however, meaningless.
To avoid such a useless solution, we normalize a domain-attribute vector and add a constraint such that an obtained solution is also normalized.

More specifically, we first preprocess all the continuous domain-attribute vector in training data such that $||\ba||_{2}^{2}=1$,\footnote{
If the domain-attribute vector consists of continuous and (after-mentioned) categorical variables, we split the vector into a continuous and a categorical domain-attribute part and apply normalization to the continuous part.
} and then add $\|\ba\|_{2}^{2}=1$ as constraints to an optimization problem.
As long as the objective function and other constraints are convex, the optimization problem with the constraint $\|\ba\|_{2}^{2}=1$ is still a convex optimization problem.
To be precise, the optimization problem is \emph{second-order cone programming} \cite{BOOK:Boyd+Vandenberghe:2004}.

\paragraph{Categorical domain attribute:}
In a number of applications, we may want to use a categorical type of variable as a domain attribute.
We can encode such a categorical variable to a $0$-$1$ vector, called the encoded domain attribute, as a \emph{binary} domain attribute.
However, optimization over binary variables results in a $0$-$1$ integer programming, which is NP-hard.
We thus use a relaxation technique to avoid solving NP-hard optimization problems.

If $m$ categories are encoded as an $m$-dimensional domain-attribute vector, choosing $k$ from $m$ categories can be expressed as convex constraints: $0\leq a_{j}^{}\leq 1$ and $\sum_{j=1}^m a_{j}^{}=k$.
After solving the optimization problem, we round the obtained solution to binary variables.
By the relaxation, we can handle categorical domain attributes under a convex optimization framework.

\paragraph{Budget limitation:}
In practice, we may need to pay attention to the cost of a domain-attribute variable.
Suppose that a domain-attribute vector is element-wise non-negative, i.e., $a_{j}^{}\geq0,\;j=1,\ldots,m$.
Let $C$ be a total budget and $\bgamma\in\Re_{}^{m}$ be a cost vector whose each element is a cost of using a corresponding domain-attribute variable.
To reflect budget limitation, we add the convex constraint $\ba^\top\bgamma\leq C$ as a budget limitation to an optimization problem.

\section{Interaction of domain attributes}
\label{sec:sdp}
Interaction of domain attributes, i.e., the dependency of variables, is important, in practice. 
In this section, we explain the means to handle interactions of domain attributes.

\paragraph{LAM with domain-attribute interaction:}
One approach is to make an interaction term, e.g., $a_{j}^{}a_{j'}^{}$ for $j\neq j'$, and use the extended domain-attribute vector by concatenating these interaction terms with the original domain-attribute vector.
Then, we can use the LAM, meaning that Proposition~\ref{thm:ccv} holds.
For example, if we are interested in the second-order interaction only, the extended domain-attribute vector is simply expressed as $\barba\otimes \barba$, where $\barba:=(\ba_{}^{\top}, 1)_{}^{\top}$ and $\otimes$ denotes the Kronecker product\footnote{
For vectors $\ba\in\Re_{}^{m_{1}^{}}$ and $\bb\in\Re_{}^{m_{2}^{}}$, $\ba\otimes\bb=(a_{1}^{}\bb_{}^{\top},a_{2}^{}\bb_{}^{\top}, \ldots, a_{m_{1}^{}}^{}\bb_{}^{\top})_{}^{\top}\in\Re_{m_{1}^{} m_{2}^{}}$.}.
The prediction model taking the second-order interaction into account is still linear in the extended domain-attribute vector $\ba\otimes\ba$; by redefining $\ba=\barba\otimes\barba$, $g(\bx,\ba)=\ba_{}^{\top}\barbphi(\bx)$, where $\barbphi(\bx):=(\phi_{1}^{}(\bx),\ldots, \phi_{(m+1)_{}^{2}}^{}(\bx))_{}^{\top}\in\Re_{}^{(m+1)_{}^{2}}$.
Although the LAM can be extended to handle domain-attribute interactions, we use the LAM without domain-attribute interactions in our experiments.
This is because we next develop a model for handling the interaction of domain attributes efficiently.

\paragraph{Semidefinite programming approach:}
In Section~\ref{sec:constraints}, we explained the relaxation approach for a binary domain attribute, i.e., the constraint $a\in\{0,1\}$ is relaxed to $0\leq a \leq1$.
While the relaxation approach is useful, if our interest is the second-order interaction of binary domain attributes, we can use the theoretically-supported algorithm \cite{KDD:Ito+Fujimaki:2017} inspired by the Goemans and Willamson's MAX-CUT approximation algorithm \cite{JACM:Goemans+Willamson:1995}.

Since the original purpose of the algorithm developed by \cite{KDD:Ito+Fujimaki:2017} is for price optimization, we modify it for the domain attribute optimization problem.
To this end, let us first define the model taking into account the second-order binary domain-attribute interaction as
%\begin{Definition}[Quadratic-in-binary-attribute model (QBM)]
\begin{Definition}[Quadratic-in-binary-attribute model]
A quadratic-in-binary-attribute model (QBM) is defined as 
\begin{align*}
g(\bx,\ba)=\sum_{j=1}^{m}\sum_{j'=1}^{m} a_{j}^{}a_{j'}^{}
    \Big(
    \sum_{q=1}^{r}\phi_{q}^{(j)}(\bx)\phi_{q}^{(j')}(\bx)
    \Big),
\end{align*}
where $\phi_t\colon\Re_{}^{d}\to\Re$ is a basis function to be learned and $r$ is a positive integer that controls flexibility of the QBM, similarly to factorization machines \cite{ICDM:Rendle:2010}.
Note that in the binary domain attribute case, a linear term, i.e., $\ba_{}^{\top}\bphi(\bx)$, is included in the quadratic term because of $a_{j}^{2}=a_{j}^{}$ for any $j\in\{1,\ldots,m\}$.
\end{Definition}
We can check that the QBM can be expressed as $g(\bx,\ba)=\ba_{}^{\top}\bPhi(\bx)\ba$, where $[\bPhi(\bx)]_{j,j'}^{}=\bphi_{j}^{}(\bx)_{}^{\top}\bphi_{j'}^{}(\bx)$, $\bphi_{j}^{}(\bx)=\big(\phi_{j}^{(1)}(\bx),\ldots,\phi_{j}^{(r)}(\bx)\big)_{}^{\top}\in\Re_{}^{r}$.
%$\vec$ the vectorization of a matrix\footnote{
%For a matrix $\bA\in\Re\x{}{m\x{1}{}\times m\x{2}{}}$,
%$\vec\big(\bA\big)=(a\x{1,1}{},a\x{1,2}{},\ldots,
%a\x{m\x{1}{},m\x{2}{}-1}{},a\x{m\x{1}{},m\x{2}{}}{})\x{}{\top}
%\in\Re\x{m\x{1}{}m\x{2}{}}{}$.
%}.
%This shows that interactions of attributes can be handled by
%linear-in-attribute models.
The QBM with linear constraints is binary quadratic programming (BQP), which is difficult to optimize in general.
However, for certain special cases, we can solve BQP efficiently.

We next introduce a constraint to domain-attribute vectors:
\begin{Definition}[Choice constraint]
For a set of indices $\cI$, a choice constraint is $\sum_{j\in\cI}^{}a_j=1$.
\end{Definition}
Hereafter, if we use the choice constraints, we assume that the index sets for the choice constraints satisfies the following condition:
The index sets $\{\cI_{p}^{}\}_{p=1}^{m}$ satisfies $\sum_{p=1}^{m}|\cI_{p}^{}|=m$, where $|\cdot|$ denotes the size of a set.
%we assume that the index sets for the choice constraints
%cover all the indices of the attribute vector and 
%these are are disjoint:
%$\bigcup\x{p=1,\ldots,q}{}\cI\x{j}{}=\{ 1,\ldots,m \}$,
%and $\bigcap\x{p=1,\ldots,q}{}\cI\x{j}{}=\emptyset$.

For the mean response and the QBM with the choice constraints, the BQP optimization problem can be relaxed to the following \emph{semidefinite programming} (SDP):
\begin{align}
\begin{aligned}
&\maximize_{\bY\in\mathbb{S}_{+}^{m+1}}^{}\,
    && \tr\big(\bC_{}^{\top}\bY\big) \\
&\subjectto && Y_{j,j}^{}=1,\;j=1,\ldots,m+1 \\
& && \sum_{j\in\cI_{p}^{}}^{}Y_{j,m+1}^{}=2-|\cI_{p}^{}|,\; p=1,\ldots,m \\
& && \sum_{j\in\cI_{p}^{}}^{}\sum_{j'\in\cI_{p'}}^{} Y_{j,j'}^{}
 = \big( 2-|\cI_{p}^{}| \big) \big( 2-|\cI_{p'}^{}| \big), \\
& && \quad p,p'=1,\ldots,m ,
\end{aligned}
\label{eq:sdp}
\end{align}
where $\mathbb{S}_{+}^{m+1}$ is the set of real symmetric semidefinite matrices of size $m+1$ \cite{BOOK:Boyd+Vandenberghe:2004}, 
\begin{align*}
\bC:=\frac{1}{4}
\begin{pmatrix}
\barbC & \barbC\bone_{m}^{} \\
\bone_{m}^{\top}\barbC & \bone_{m}^{\top}\barbC\bone_{m}^{} 
\end{pmatrix} \in \mathbb{S}_{+}^{m+1},
\end{align*}
$\barbC:=(1/|\cB|)\sum_{\bx\in\cB}^{}\sum_{t=1}^{r} \phi_{r}^{(t)}(\bx)\phi_{r'}^{(t)}(\bx)$, and $\bone_{m}^{}$ is the $m$-dimensional all-ones vector.
We round the obtained solution to the binary domain-attribute vector on the basis of the randomized search algorithm proposed in \cite{KDD:Ito+Fujimaki:2017}.
% The derivation is deferred to Appendix.

An advantage of the SDP formulation is that if we further add convex constraints to the optimization problem in Eq.~\eqref{eq:sdp}, the optimization problem is still SDP.
Thus, thanks to the SDP formulation, we can obtain the solution efficiently rather than solving the BQP.
%\todo[inline]{large scale model}
%\todo[inline]{discussion about CVaR}

%\paragraph{Quadratic-in-binary-attribute model:}
%\begin{definition}[Binary attribute]
%A binary attribute takes $0$ or $1$. 
%The domain of a binary attribute vector is $\{ 0, 1 \}\x{}{m}$.
%\end{definition}

\section{Theoretical analyses}
\label{sec:theory}
In this section, we present two theoretical properties of our proposed algorithm.
Specifically, in Section~\ref{sec:gen_err_bound}, we show generalization error bounds of the prediction method used in Section~\ref{sec:prediction}.
In Section~\ref{sec:approx_factor}, we present approximation factors of the optimization method used in Section~\ref{sec:optimization}.
All the proofs are given in Appendix~\ref{sec:proofs}.

%\subsection{Generalization error bound}
\subsection{Generalization error bound}
\label{sec:gen_err_bound}
In this analysis, we consider the case where the feature mapping function can be expressed as $\bphi(\bx) = \bW \bpsi(\bx)$, where $\bW \in \bbR^{m \times b}$ is the parameter matrix, $\bphi\colon\Re^d\to\Re^b$ is the vector of basis functions, i.e., $\bpsi(\bx)=(\psi_1(\bx), \ldots, \psi_b(\bx))^\top$, and $\{ \psi_\ell\colon\Re^d\to\Re \}_{\ell=1}^b$ is fixed in advance.
Then, LAM can be expressed as the bilinear function model $g(\bx, \ba)=\ba^\top\bW \psi(\bx)$.

The key idea is to reformulate the prediction model $g$ into 
\begin{align*}
h(\bxtilde)=\bwtilde^\top\bxtilde,    
\end{align*}
where $\bwtilde:=\vec(\bW^\top) \in \bbR^{dm}$, $\bxtilde_{i'}:=\vec( \bpsi(\bx_i^{(t)}) (\ba^{(t)})^\top ) \in \bbR^{dm}$, and $i'=i+\sum_{t'=1}^{t-1} n^{(t')}$.
Accordingly, we express a set of training samples drawn from a distribution $Q$ as
$\{ (\bxtilde_i, y_i) \}_{i=1}^{n}$, where $n=\sum_{t=1}^{T} n^{(t)}$.
We assume that there exists the target labeling function $f_Q\colon\bbR^{dm}\to\bbR$, $y_i=f_Q(\bxtilde_i)$.
We next respectively define the expected and empirical risks as $J_Q(h):=\E_{Q}\big[ \ell\big( h(\bxtilde), f_Q(\bxtilde) \big)\big]$ and $\Jh_Q(h):=\frac{1}{n}\sum_{i=1}^n \ell\big( h(\bxtilde_i), f_Q(\bxtilde_i) \big)$,
%\begin{align*}
%J_Q(h)&:=\E_{Q}\big[ \ell\big( 
%    h(\bxtilde), f_Q(\bxtilde)
%    \big)\big] , \\
%\Jh_Q(h)&:=\frac{1}{n}\sum_{i=1}^n \ell\big( 
%    h(\bxtilde_i), f_Q(\bxtilde_i)
%    \big) ,
%\end{align*}
where $E_{Q}$ is the expectation over $Q$.

We also assume the $\ell_q$ loss function $\ell_q(y,y')=| y - y' |^q$ for a real number $q\geq 1$.
Besides, we assume that there exists a $M>0$ such that $| h(\bxtilde) - f_Q(\bxtilde) | \leq M$ for all $\bxtilde$ and $h\in\cH$, and there exists $B_{\bwtilde}>0$ such that $\| \bwtilde \| \leq B_{\bwtilde}$. 
Similarly, we assume that $\| \bx \| \leq B_\bx$ and $\| \ba \| \leq B_\ba$, leading to $\| \bxtilde \| \leq B_\bx B_\ba$.
We denote a function class of $g$ by $\cH:=\{ h(\bxtilde)=\bwtilde^\top\bxtilde \mid \|\bwtilde\|_2 \leq B_{\bwtilde}, \|\bxtilde\|_2 \leq B_\bx B_\ba \}$.

Fix $\ba$, then, we have the following proposition:
\begin{proposition}
\label{thm:std_gen_err}
Fix $\ba$, then, for any $\delta > 0$, the following inequality holds with probability at least $1-\delta$ for any $h \in \cH$:
\begin{align*}
J_Q(h) \leq \Jh_Q(h) + \frac{ 2qM^{q-1}B }{ \sqrt{n} } + M^q\sqrt{ \frac{ \ln(1/\delta) }{ 2n } },
\end{align*}
where $B:=B_{\bwtilde} B_\bx B_\ba$.
\end{proposition}
This result shows that for the same domain, i.e., the domain characterized by the training attributes, the generalization error bound converges with the order $\cO(1/\sqrt{n})$, which is the optimal without any additional assumption~\citep{TIT:Mendelson:2008}.
%This result shows that for the same domain, characterized by a fixed $\ba$, the generalization error bound converges with the order $\cO(1/\sqrt{n})$, which is the optimal without any additional assumption~\citep{TIT:Mendelson:2008}.

%We next consider the case when a target domain differs from a source domain.
%We regard that $\bxtilde_{i'}=\vec( \bpsi(\bx_i^{(t)}) (\ba^{(t)})^\top )$ and $\bxtilde_{i}'=\vec( \bpsi(\bx_i^{(t)}) (\ba^{(t')})^\top )$ are drawn from distributions $Q$ and $P$, respectively, where $\ba_{(t')}$ is an unseen domain-attribute vector.
Next, we consider generalization error bounds on the domain characterized by the optimized attribute vectors.
Compared with the above analysis, it requires to measure the difference between the source (training attribute vectors) and target (optimized attribute vectors) domains.
Let $P$ and $Q$ be distributions of the target and source domains, respectively.
We regard that $\bxtilde_{i'}=\vec( \bpsi(\bx_i^{(t)}) (\ba^{(t)})^\top )$ and $\bxtilde_{i}'=\vec( \bpsi(\bx_i^{(t)}) (\ba^{(t')})^\top )$ are independently drawn from distributions $Q$ and $P$, respectively, where $\ba_{}^{(t')}$ is an test/optimized domain-attribute vector.
We then analyze the generalization error bounds of ZSDA on the basis of a tool for domain adaptation.
More specifically, we have training samples from the source domain with distribution $Q$ and evaluate the performance on the target domain with distribution $P$.

We first define 
$L_Q(h,h'):=\E_Q\big[ \ell\big( h(\bxtilde), h'(\bxtilde)\big)\big]$ and 
$L_P(h,h'):=\E_P\big[ \ell\big( h(\bxtilde), h'(\bxtilde)\big)\big]$,
and denote the corresponding empirical approximators by
$\Lh_Q(h,h')$ and $\Lh_P(h,h')$, respectively.
By definition, we have $J_Q(h)=L_Q(h,f_Q)$ and $\Jh_Q(h)=\Lh_Q(h,f_Q)$.
Similarly, $J_P(h):=\E_{P}\big[ \ell\big( h(\bxtilde), f_P(\bxtilde) \big) \big]$ and $\Jh_P(h):=\frac{1}{n'}\sum_{i=1}^n \ell\big( h(\bxtilde_i), f_P(\bxtilde_i) \big)$, where $n'$ is the number of samples in the domain $P$ and $f_P$ is the target labeling function in $P$.
To measure the difference between two distributions $P$ and $Q$, we use the \emph{discrepancy distance} \citep{COLT:Mansour+etal:2009} defined as
\begin{align*}
\disc(P, Q)=\sup_{ h,h'\in\cH }\; \big|
    L_P(h, h') - L_Q(h, h') \big| .
\end{align*}

Let $h_Q^\ast$ and $h_P^\ast$ be the minimizers of $J_Q(h)$ and $J_P(h)$, respectively.
We assume that  $| h(\bxtilde) - h'(\bxtilde) | \leq M$ for all $\bxtilde$ and $h,h'\in\cH$.
Under the above assumptions, we have the following proposition:
\begin{proposition}
\label{thm:gen_err}
For any $\delta > 0$, the following inequality holds with probability at least $1-\delta$ for any $h \in \cH$:
\begin{align*}
L_P(h,f_P) &\leq \Lh_Q(h, h_Q^\ast) + \disc(\Ph, \Qh) 
  + L_P(h_Q^\ast, h_P^\ast) \\
&\quad + L_P(h_P^\ast, f_P) + \frac{ 10qM^{q-1} B }{ \sqrt{n} }
+ 3M^q\sqrt{ \frac{ \ln(3/\delta) }{ 2n } }.
\end{align*}
\end{proposition}
This result shows that the target expected error $L_P$ in the left-hand side is upper-bounded by the source empirical error $\Lh_Q$ plus additional constants and confidence terms.

Note that the same transformation was used in~\citep{ICML:Romera-Paredes+Torr:2015} to analyze generalization error bounds, but they considered classification while our focus is regression. 
Moreover, the subsequent analyses differ from \citep{ICML:Romera-Paredes+Torr:2015}, e.g., a measure for domains.

%\subsection{Approximability in optimization step}
\subsection{Approximation factor}
\label{sec:approx_factor}
%\subsection{Approximation ratio}
We here analyze the attribute optimization step in our framework
from a complexity-theoretic point of view.
In a nutshell,
we show that
the non-negative linear counterpart of attribute optimization is a generalization of packing integer programs \cite{JCSS:Raghavan:1988}, and
it is \textbf{NP}-hard in general but
approximable if the vectors representing constraints are ``sparse.''
For the sake of simplicity of analysis, we make the following assumptions:
\begin{itemize}
    \item The objective function is given by $f(\ba) = F(g_{\ba})$, where
    $F$ is a mean response and $g_{\ba}$ follows a LAM model; i.e.,
    $f(\ba) = \frac{1}{|\cB|}\sum_{\bx \in \cB} \ba^{\top} \bphi(\bx)$ for
    a non-negative vector function $\bphi(\cdot)$.
    \item Constraint functions $b_j$ and $c_j$ are \emph{non-negative and linear}; i.e.,
    for each $j \in [s]$, there exists a vector $\bb_j \in \Re^{n}_{\geq 0}$ and a scalar $d_j \geq 0$ such that $b_j(\ba) = \bb_j^\top \ba - d_j$, and
    for each $j \in [t]$, there exists a vector $\bc_j \in \Re^{n}_{\geq 0}$ and a scalar $e_j \geq 0$ such that $c_j(\ba) = \bc_j^\top \ba - e_j$.
\end{itemize}
The attribute optimization problem under the above assumptions (hereafter called \emph{non-negative attribute optimization}; NAO)
can be written as follows:
\begin{align}
\label{eq:nao}
    \max_{\ba}\!\left\{\! \frac{1}{|\cB|} \! \sum_{\bx \in \cB} \! \ba^{\!\top} \! \bphi(\bx) \mid
    \bb_j^\top \ba \leq d_j, \forall j \in [s]; \bc_j^\top \ba \leq e_j, \forall j \in [t]
    \! \right\}\!.
\end{align}
Hereafter, Eq.~\eqref{eq:nao} is referred to as \texttt{NAO\_mix}
if $\ba$ contains both binary and real-valued variables, and
\texttt{NAO\_01} if $\ba$ contains only binary variables.
\texttt{NAO\_01} is a special case of \texttt{NAO\_mix}.

We will discuss the relation between NAO and
a discrete optimization problem called \emph{packing integer programs} (PIPs) \cite{JCSS:Raghavan:1988}.
Our first result is that \texttt{NAO\_01} includes PIPs as a special case,
implying a hardness-of-approximation result of \texttt{NAO\_01}.

\begin{theorem}
\label{thm:pip-hard}
There exists a polynomial-time reduction from PIPs to \texttt{NAO\_01}.
It is thus \textbf{NP}-hard to approximate \texttt{NAO\_01} within a factor of $n^{1-\epsilon}$ for any $\epsilon > 0$, where $n$ is the dimension of a domain-attribute vector.
\end{theorem}

Having known that \texttt{NAO\_01} is hard in general,
we restrict the class of input structures to study the approximability of \texttt{NAO\_mix}.
For each $i \in [n]$,
let $S(i)$ be the number of constraints that $i$ appears in;
i.e.,
\begin{align*}
S(i) := |\{ j \in [s] : c_{j,i} > 0 \}| + |\{j \in [t] : d_{j,i} > 0\}|.
\end{align*}
The \emph{column sparsity} is then defined as $S := \max_{i \in [n]} S(i)$.
Our second result states that
we can approximate \texttt{NAO\_mix} accurately if $S$ is bounded.

\begin{theorem}
\label{thm:pip-approx}
There exists a polynomial-time $8S$-factor approximation algorithm for \texttt{NAO\_mix},
where $S$ is the column sparsity of an input instance.
\end{theorem}
The above theorem means that
we can find a domain-attribute vector that has an objective value at least $\frac{1}{8S}$ times the optimum of Eq.~\eqref{eq:nao} in polynomial time.
Practically, we can assume the column sparsity $S$ to be small; e.g.,
in experimental evaluation, $S$ is at most $4$.
%\todo[inline]{$***$ -> 2 or 3?}
Our approximation algorithm makes use of an approximation algorithm for PIPs by~\citet{TheoComp:Bansal+etal:2012},
which conceptually works as follows:
(1) it solves the linear programming relaxation of PIPs and
applies a randomized rounding on the obtained solution $\tilde{\ba}$ to compute a binary vector $\ba$, and
(2) it repeatedly picks an index $i$ with the largest entry $\tilde{a}_i$ that violates some constraints and assigns $a_i \leftarrow 0$.
Note that in practice, the above procedure is equivalent to the simple rounding method described in Section~\ref{sec:constraints}.

%involves complex randomized rounding procedures;
%further, we need to handle real-valued (categorical) variables to
%devise approximation guarantees.
%Note that we do not employ the algorithm in experimental evaluation
%because a simple rounding method works pretty well in practice.

\section{Experiments}
\label{sec:experiments}
In this section, we report experimental evaluations of
our proposed method on both toy datasets and real-world datasets.

\subsection{Common settings}
We first describe the common settings between both toy and real-world datasets.

\paragraph{Evaluation:}
To evaluate the performance, we used a domain-wise dataset split, i.e., domain-attribute vectors for testing do not appear in training.
Let $\cA_\tr^{}:=\{ \ba_{}^{(\kappa(t))} \}_{t=1}^{T_\tr^{}}$ and $\cA_\te^{}:=\{ \ba_{}^{(\kappa(T_\tr^{}+t'))} \}_{t'=1}^{T_\te^{}}$, where $\kappa$ maps the original index to the permuted index, and $T_\tr^{}$ and $T_\te^{}$ respectively denote the number of training and test domains.
The dataset $\cD=\{ (\cD_{}^{(t)}, \ba_{}^{(t)}) \}_{t=1}^T$ was split into $\cD_\tr:=\{ (\cD_{}^{(\kappa(t))}, \ba_{}^{(\kappa(t))}) \}_{t=1}^{T_\tr^{}}$ and $\cD_\te:=\{ (\cD_{}^{(\kappa(T_\tr^{}+t'))}, \ba_{}^{(\kappa(T_\tr^{}+t'))}) \}_{t'=1}^{T_\te^{}}$ such that $\cA_\tr^{} \cap \cA_\te^{} = \emptyset$.
In our experiment, we split a dataset into $80\%$ training and $20\%$ test datasets.

For each test set of features $\cB_\te^{(t')}:=\{ \bx_i^{(t')} \}_{i=1}^{n_{}^{(t')}}$ for test domain $t'$, we optimized the domain attributes, denoted by $\bah_{}^{(t')}$.
Let $g_{}^\ast$ be an evaluation function appropriately defined for each experiment.
In the artificial datasets, $g_{}^\ast$ was the ground-truth function used for data generation. 
In the real-world datasets, we regarded the LAM trained with ``all'' the samples, including the samples from the training and test domains, as $g^\ast$ since we cannot access $g^\ast$.

With the optimized domain-attribute vector $\bah_{}^{(t')}$, we computed the average and relative standard deviation\footnote{
The relative standard deviation is the standard deviation over $\{ g_{}^\ast(\bx_i^{(t')}, \bah_{}^{(t')}) \}_{i=1}^{n_{}^{(t')}}$ divided by the mean $\overline{g}_{}^{(t')}$.
} over $\{ g_{}^\ast(\bx_i^{(t')}, \bah_{}^{(t')}) \}_{i=1}^{n_{}^{(t')}}$, denoted by $\overline{g}_{}^{(t')}$ and $\widetilde{g}_{}^{(t')}$, respectively.
As evaluation metrics, we used $\overline{g}_\te=(1/T_\te^{})\sum_{t'=1}^{T_\te^{}} \overline{g}_{}^{(\kappa(T_\tr^{} + t'))}$ and $\widetilde{g}_\te=(1/T_\te^{})\sum_{t'=1}^{T_\te^{}} \widetilde{g}_{}^{(\kappa(T_\tr^{} + t'))}$.

\paragraph{Prediction models:}
For comparison, we used the FUA with linear ridge regression.
Specifically, we trained $s=\bw_{}^\top\ba + b$ with $\{ (\ba_{}^{(t)}, \overline{y}^{(t)} \}_{t=1}^{T_\tr^{}}$, where $b\in\bbR$ is the intercept.
We then maximized $\bah_{}^\top\bwh$ in a feasible domain, where $\bwh$ is the estimated parameter.
We refer to this method as \emph{FUA-Mean}.
The above attribute optimization can be regarded as the use of the mean gain function in our proposed method, but it does not take the dependency of features into account in both the prediction and optimization steps.
Due to the FUA's nature, we cannot use the CVaR-based gain function in the same way as our proposed method.

We used a four-layer neural network ($d$-$100$-$100$-$m$) for the feature mapping function $\bphi$ of the LAM.
For QBM, we used a three-layer neural network ($d$-$100$-$m$) for $\bphi_q$ and set $r=3$.
For the hidden layers of neural networks, we used the ReLU activation function \citep{AISTATS:Glorot+etal:2011} and \emph{batch normalization} \citep{ICML:Ioffe+Szegedy:2015}.
We further split the training data into $80\%$ training and $20\%$ validation data.
We then trained the neural network with the \emph{Adadelta} optimizer \citep{arXiv:Zeiler:2012} until $200$ epochs, and we used the model that achieved the lowest validation error for inference.
We refer to LAM/QBM with the mean gain function as \emph{LAM-Mean}/\emph{QBM-Mean} and refer to LAM with the CVaR gain function and $\beta=0.05$ as \emph{LAM-CVaR}.

\subsection{Mean vs. CVaR-based gain function}
\label{sec:exp_lin}
We here show the effect of the mean and CVaR-based gain functions.

\paragraph{Data:}
We considered an attribute vector to consist of continuous ($a_{1}^{},a_{2}^{},a_{3}^{}$) and binary ($a_{4}^{},a_{5}^{},a_{6}^{}$) variables, i.e., $m=6$.
Each element of the continuous variables was drawn from the standard Gaussian distribution, denoted by $\cN(0,1_{}^{2})$, and the continuous variables were then normalized such that $\sum_{i=1,2,3}^{}a_{i}^{2}=1$.
For the categorical variables, one of the elements was chosen uniformly at random.

For a response function, we used $g_{}^\ast(\bx,\ba)=\bx_{}^{\top}\bW\ba$, where each element of $\bW\in\Re_{}^{d\times m}$ was drawn from $\cN(0,1_{}^2)$.
The response $y$ was then observed by $y=g_{}^\ast(\bx,\ba)+\varepsilon$, where $\varepsilon$ was drawn from $\cN(0, 0.1_{}^{2})$.
We drew $15$ attribute vectors and $50$ samples for each object, i.e., we had $\{(\cD_{}^{(t)},\ba_{}^{(t)})\}_{t=1}^{15}$, where $\cD_{}^{(t)}=\{(\bx_{i}^{(t)},y_{i}^{(t)})\}_{i=1}^{50}$.
We set $d$ as $10$ and drew each element of $\bx$ from $U(0,1)$, where $U(a,b)$ denotes a uniform distribution with the range $[0, 1]$.

\begin{table}[t]
\centering
\caption{
Average of $\overline{g}^{}_\te$ over $10$ trials. The number in parentheses is average of $\widetilde{g}_\te^{}$.
In terms of the average response, the mean gain function was better than the CVaR-based gain function.
The CVaR-based gain function was stable in terms of the relative standard deviation of obtained responses (see also Figure~\ref{fig:dist_mean_cvar}).
The boldface denotes the best method in terms of the average gain.
}\label{tab:synth}
\begin{tabular}{rrr}
\toprule
\multicolumn{1}{c}{FUA-Mean}
  & \multicolumn{1}{c}{LAM-Mean}
  & \multicolumn{1}{c}{LAM-CVaR} \\
\midrule
3.24 (0.54) & \textbf{3.28 (0.52)} & 2.86 (0.46) \\
\bottomrule
\end{tabular}
\end{table}

\begin{figure}[t]
    \centering
    \includegraphics[clip,width=.45\columnwidth]{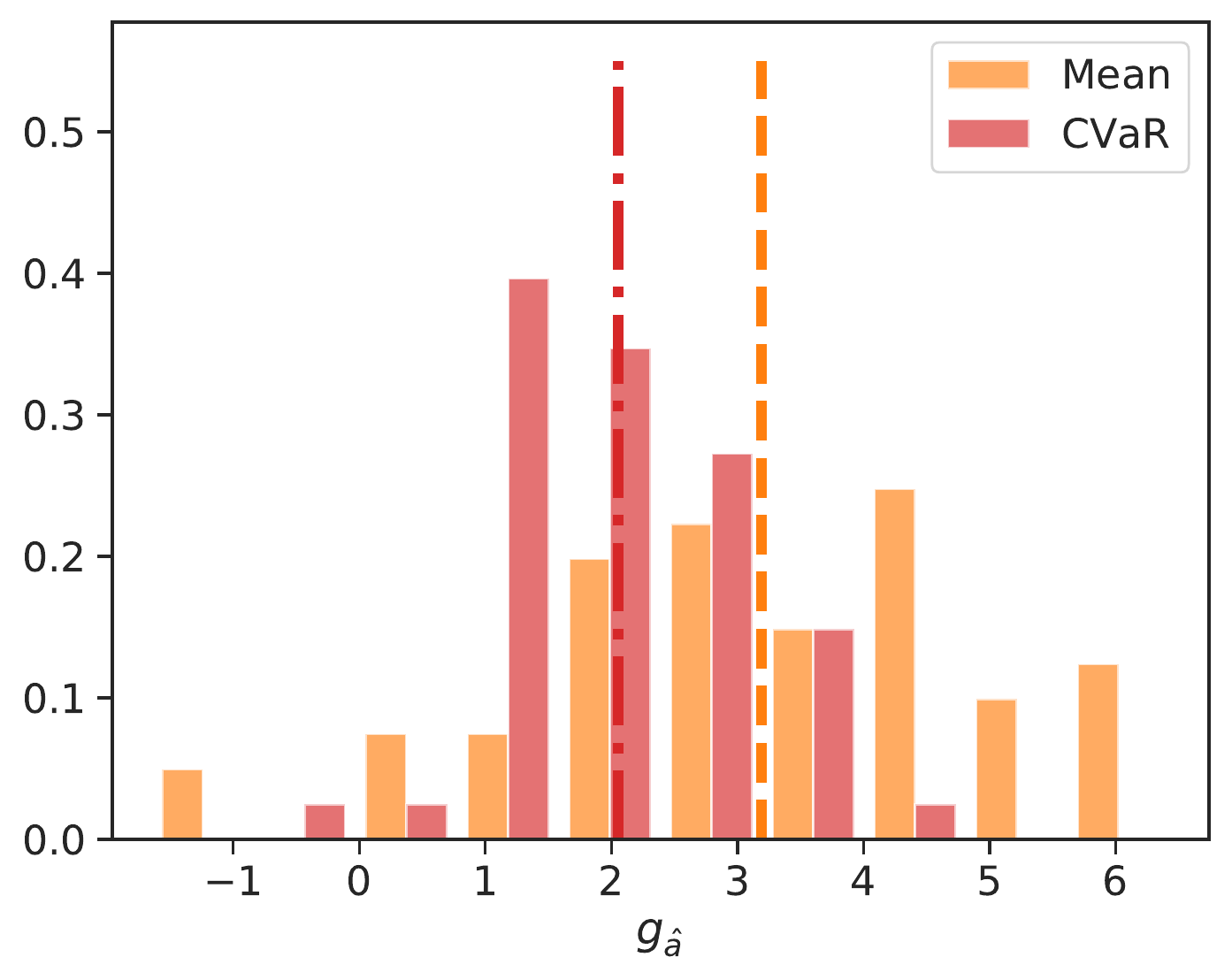}
    \caption{
    Normalized histograms of $g(\bx, \bah^{ (t') } )$ obtained by the mean and CVaR-based gain functions, respectively.
    We picked a test domain and compared the optimized domain-attribute vectors of the mean and CVaR-based gain functions.
    The vertical lines indicate the average values of the two responses.
    Since the CVaR-based gain function takes the response distribution into account, in particular, left tail of the distribution, the minimum value obtained by this function was higher than the mean gain function.
    Also, the response distribution obtained by the CVaR-based gain function was narrower than that obtained by the mean gain function.
    }\label{fig:dist_mean_cvar}
\end{figure}

\paragraph{Results:}
Table~\ref{tab:synth} summarizes the averages of $\overline{g}^{}_\te$ and $\widetilde{g}_\te^{}$ over $10$ trials, showing that the mean gain function was better than the CVaR-based gain function in terms of the average response, while the latter gain function was stable in terms of standard deviation.

To visualize why the standard deviation of the CVaR-based gain function was lower than that of the mean gain function, we plotted two histograms of the obtained responses of the attributes obtained by the mean and CVaR-based gain functions, respectively, in Figure~\ref{fig:dist_mean_cvar}, which illustrates that the response distribution obtained by the CVaR-based gain function was narrower than that obtained by the mean gain function. 
This is because the CVaR-based gain function took into account the response distribution, in particular, the left-tail of the distribution, resulting in the relative standard deviation of the obtained responses being smaller while the average response remained high, almost comparable to that obtained by the mean gain function in this synthetic data experiment.

%Also, we show the densities estimated by the kernel density estimation %\cite{Book:Silverman:1986}.

\subsection{Effect of interaction of domain attributes}
\label{sec:exp_fm}

\paragraph{Data:}
We considered a $10$-dimensional attribute vector to consist of binary variables with $\sum_{j=1}^3 a_j = 1$, $\sum_{j=4}^6 a_j = 1$, and $\sum_{j=7}^{10} a_j = 1$.
The categorical attribute was chosen from each group uniformly at random.
For a response function, we used $g_{}^\ast(\bx,\ba)=(1/30)\ba_{}^{\top} \big( \sum_{q=1}^3 \phi_q^{}(\bx) \phi_q^{}(\bx)_{}^\top \big) \ba$, where $\phi_q(\bx):=\bW_q\bx$, and each element of $\bW_q\in\Re_{}^{m\times d}$ was drawn from $\cN(0,1_{}^2)$.
The response $y$ was then observed by $y=g_{}^\ast(\bx,\ba)+\varepsilon$, where $\varepsilon$ was drawn from $\cN(0, 0.1_{}^{2})$.
We drew $15$ attribute vectors and $50$ samples for each object, i.e., we had $\{(\cD_{}^{(t)},\ba_{}^{(t)})\}_{t=1}^{10}$, where $\cD_{}^{(t)}=\{(\bx_{i}^{(t)},y_{i}^{(t)})\}_{i=1}^{50}$.
We set $d$ as $10$ and drew each element of $\bx$ from $U(0,1)$, where $U(a,b)$ denotes a uniform distribution with the range $[0, 1]$.

\paragraph{Results:}
Table~\ref{tab:synth_fm} summarizes the averages of $\overline{g}^{}_\te$ and $\widetilde{g}_\te^{}$ over $10$ trials, showing that the obtained gain of QBM was larger than that of FUA and LAM.
This result shows that when there is a dependency between domain-attribute variables, models incorporating such a domain-attribute interaction attain higher performance for attribute optimization.

\begin{table}[t]
\centering
\caption{
Averages of $\overline{g}^{}_\te$ over $10$ trials. 
The number in parentheses is the average of relative standard deviation $\widetilde{g}_\te^{}$.
Unlike LAM, QBM attained much higher gain.
The boldface denotes the best method in terms of the average gain.
}\label{tab:synth_fm}
\begin{tabular}{rrr}
\toprule
\multicolumn{1}{c}{FUA-Mean}
  & \multicolumn{1}{c}{LAM-Mean}
  & \multicolumn{1}{c}{QBM-Mean} \\
\midrule
2.06 (0.94) & 2.08 (0.93) & \textbf{2.20 (0.97)} \\
\bottomrule
\end{tabular}
\end{table}

\subsection{Real-world datasets}
We next evaluate the performance on benchmark datasets.
The statistics of the datasets are summarized in Table~\ref{tab:bench_stats}, and the details are given below.

\paragraph{Sushi:}
The SUSHI\footnote{
Sushi is a Japanese dish containing vinegared rice.
} Preference (Sushi) Dataset~\citep{KDD:Kamishima:2003} consists of consumer ratings for sushi, features of consumers, and domain attributes of each kind of sushi.\footnote{
\url{http://www.kamishima.net/sushi/}
}
The rating of sushi is done by five-grade evaluation (from $0$ to $4$), the mean rating is $2.73$, and the median rating is $3.00$.
For the description (domain attributes) of sushi, we used the style of sushi, the type of sushi, the oiliness, and the normalized price.
In this dataset, the type and style of sushi are categorical domain-attributes, and the oiliness and normalized price are continuous domain-attributes.
The task was to find better combinations of domain attributes of sushi.

For the consumer features, we used gender, range of ages, the prefecture in which until $15$ years the consumer had lived longest, the prefecture in which the consumer currently lives, and the total time taken for stating their preference of sushi.
We then converted the characteristics introduced above into numerical vectors and finally obtained $12$-dimensional feature vectors and $16$-dimensional domain-attribute vectors.
% The precise preprocessing is detailed in Appendix~\ref{app:sushi_preprop}.

\paragraph{Coffee:}
The coffee quality dataset contains $1338$ reviews.\footnote{
The dataset was downloaded from \url{https://github.com/jldbc/coffee-quality-database}.
Note that we deleted samples that have missing entries.
}
Reviews are given for beans and farms. 
We used the information for a farm as features and that for a bean as domain attributes.
Specifically, the ``Country of Origin,'' ``Certification Body,'' and ``Altitude''\footnote{
If the altitude is given as a range, e.g., $1200$-$1300$, we used the mean value.
} in the dataset were used as features, and the ``Species,'' ``Processing Method,'' and ``Variety'' were used as domain attributes.
The ``Total Cup Points'' were used as the score (reward).
The full score is $100$, the minimum and maximum scores in the dataset are $59.8$ and $90.6$, respectively, the mean score is $82.1$, and the median score is $82.5$.
%Table~\ref{tab:attr_desc_coffee} shows the details of domain-attributes.
The number of possible choices of domain attributes is $2$ species, $5$ processing methods, and $29$ varieties.   
The goal was to find a combination of a processing method, species, and variety of coffee for a specific farm.

\paragraph{Book:}
We used goodbooks-10k (Book).\footnote{
\url{https://github.com/zygmuntz/goodbooks-10k}.
}
The Book dataset collects ratings of books from readers.
The range of ratings is from $1$ to $10$, the mean rating is $4.68$, and the median rating is $3.0$.
Since there were items without ratings, we used mean imputation to focus on the effect of our method for simplicity.

We used ``Age'' and ``Country`` as the features of readers, and we used tags of books annotated by users in the book-rating platform as domain attributes.
We manually extracted book tags that were likely to be relevant to ratings.
Examples of extracted tags are ``biography,'' ``comedy,'' and ``fiction.''
After preprocessing, that is, one-hot encoding, we had $7{,}121$ ratings of $147$ books ($m=77$) from readers ($d=74$).

\paragraph{Results:}
Table~\ref{tab:bench_results} summarizes the averages of $\overline{g}^{}_\te$ and $\widetilde{g}_\te^{}$ over $10$ trials, showing that i) LAM with the mean gain function achieved a higher response than the other methods, and ii) LAM with the CVaR gain function tended to produce results with smaller variances among them.

In the real-world datasets, the obtained performance difference between the FUA and the proposed method was larger, compared with the artificial datasets.
Since the difference between the different sets of features in the real-world datasets was larger than that in the artificial datasets, the proposed method, taking features into account to optimize attributes, returned more suitable domain attributes than the FUA.
These results imply that attribute optimization is a promising means of cooperating with humans in designing new products.
On the basis of the results presented by our method, humans can continue further trial and error to find a better description of new products.
Another aspect of using our method is that it reduces the cost of designing products and services for each customer because our method is aware of customers' features.
The proposed method will allow us to provide products and services tailored to each customer, which will improve customer satisfaction.

\begin{table}[t]
\centering
\caption{
Statistics of real-world datasets.
}\label{tab:bench_stats}
\begin{tabular}{lrrrr}
\toprule
Dataset & $n$ & $d$ & $m$ & $T$ \\
\midrule
Sushi & 50{,}000 & 31 & 15 & 100 \\
Coffee & 1{,}161 & 63 & 36 &  23 \\
Book & 7{,}282 & 64 & 77 & 169 \\
\bottomrule
\end{tabular}
\end{table}

\begin{table}[t]
\centering
\caption{
Averages of $\overline{g}^{}_\te$ over $10$ trials. 
The Number in parentheses is the average of $\widetilde{g}_\te^{}$ over $10$ trials. 
The boldface denotes the best method in terms of the average gain.
}\label{tab:bench_results}
\begin{tabular}{lrrr}
\toprule
Dataset & \multicolumn{1}{c}{FUA-Mean}
  & \multicolumn{1}{c}{LAM-Mean}
  & \multicolumn{1}{c}{LAM-CVaR} \\
\midrule
Sushi & 3.79 (0.11) & \textbf{3.85 (0.11)} & 3.76 (0.10) \\
Coffee & 74.7 (0.09) & \textbf{99.4 (0.05)} & 98.2 (0.04) \\
Book & 4.76 (0.18) & \textbf{7.00 (0.16)} & 6.17 (0.10) \\
\bottomrule
\end{tabular}
\end{table}

\section{Conclusion}
Zero-shot domain adaptation is useful in real-world applications, e.g., predicting the sales of a new product for which labeled data are not available.
While existing studies focus on improving the prediction accuracy, we considered a reverse process for prediction that can be categorized as predictive optimization.
To this end, we proposed a simple framework for predictive optimization with zero-shot domain adaptation and analyzed the conditions in which optimization problems become convex.
Furthermore, we discussed the way of handling interactions of variables representing a domain description.
Through numerical experiments, we demonstrated the potential effectiveness of the proposed framework.

Finding a promising combination of characteristics of existing products is an important task for manufacturers.
While the amount of available data on existing products and consumer reactions is increasing day by day, handling a large amount of data is often difficult for humans without support from computer systems. 
Our simple formulation could be a guideline for investigating bottlenecks in data-driven design systems and unlock further possibilities in this direction of research.

\bibliographystyle{plainnat}
%\bibliography{bibshort,attrgen}
\bibliography{biblong,attrgen}

\clearpage
\appendix

%\subsection{Proofs}
\section{Proofs}
\label{sec:proofs}

%\subsubsection{Proofs of generalization error bound}
\subsection{Proofs of generalization error bound}
Recall the notations and assumptions:
\begin{itemize}
    \item $f_P$ and $f_Q$ are the labeling functions on the target and source domains, respectively.
    \item $\ell_q(y, y') = | y - y' |^q$ for a real number $q\geq 1$.
    \item $| h(\bxtilde) - f(\bxtilde) | \leq M$ for all $\bxtilde$ and $h \in \cH$.
    \item $| h(\bxtilde) - h'(\bxtilde) | \leq M$ for all $\bxtilde$ and $h,h' \in \cH$.
    \item $\| \bwtilde \| \leq B_{\bwtilde}$, $\| \bxtilde \| \leq B_{\bxtilde}$, $\| \bx \| \leq B_{\bx}$, and $\| \ba \| \leq B_{\ba}$.
    \item $S:=\{ (\bxtilde_i, y_i) \}_{i=1}^n \sim P$, $S':=\{ (\bxtilde'_i, y_i) \}_{i=1}^{n'} \sim Q$, and $n=n'$.
\end{itemize}

\begin{proof}%[Proof of Theorem~\ref{thm:std_gen_err}]
From the assumptions, $L_Q(h, h') = \E_Q[ \ell( h(\bxtilde), h'(\bxtilde) ) ] \leq M^q$.
Let $\cL\colon \bxtilde \mapsto \ell( h(\bxtilde), h'(\bxtilde) )$.
Based on the standard Rademacher analysis (see, e.g., \citet{book:Mohri+etal:2012}, we have for any $\delta>0$, the following inequality with probability at least $1-\delta$ for any $h, h \in \cH$:
\begin{align}
L_Q(h, h') - \Lh_Q(h, h') \leq 2\fR_n( \cL) + M^q\sqrt{ \frac{ \ln (1/\delta) }{ 2n } } ,
\label{eq:h_hp_bound}
\end{align}
where $\fRh_S(\cH):=\frac{1}{n}\E_\bsigma[ \sup_{h\in\cH}\; \sum_{i=1}^n \sigma_i h(\bxtilde_i) ]$ and $\fR_n(\cH):=\E_S[ \fRh_S(\cH)] $ are the empirical and expected \emph{Rademacher complexity}, $\sigma_i$ is an independent uniform random variables taking values in $\{+1,-1\}$, and $\bsigma:=(\sigma_1,\ldots,\sigma_n)^\top$.
Let $\cH_f=\{ \bxtilde \mapsto h(\bxtilde) - f(\bxtilde) \mid h \in \cH \}$.
Then, $\cL$ can be rewritten as $\cL=\{ \ell_q \circ h_f \mid h_f \in \cH_f \}$.
Since $\ell_q$ is $qM^{q-1}$-Lipschitz over $[ -M, M ]$, we can use Talagrand's lemma~\citep{book:Ledoux+Talagrand:1991}: $\fRh_S(\cL)\leq qM^{q-1}\fRh_S(\cH_f)$.
Furthermore, $\fRh_S(\cH_f) = \fRh_S(\cH)$.
For the linear model, the Rademacher complexity can be bounded (see, e.g., \citet{book:Mohri+etal:2012}) as $\fR_n(\cH) \leq \frac{ B_{\bwtilde} B_\bx B_\ba }{ \sqrt{ n } }$.
Replacing $h'$ with $f_Q$ in Eq.~\eqref{eq:h_hp_bound}, we have Proposition~\ref{thm:std_gen_err}. 
That is, for any $\delta>0$, the following inequality holds with probability at least $1-\delta$ for any $h\in \cH$: $L_P(h, f_P) - \Lh_P(h, f_P) \leq \frac{ 2qM^{q-1}B }{ \sqrt{n} } + M^q\sqrt{ \frac{ \ln (1/\delta) }{ 2n } }$.
%\end{proof}

%\begin{proof}[Proof of Theorem~\ref{thm:gen_err}]
Let $\cH_h=\{ \bxtilde \mapsto h(\bxtilde) - h'(\bxtilde) \mid h,h' \in \cH \}$.
$\cL$ can be rewritten as $\cL=\{ \ell_q \circ h_h \mid h_h \in \cH_h \}$.
We then have $\fRh_S(\cL) \leq qM^{q-1}\fRh_S(\cH_h) \leq 2qM^{q-1}\fRh_S(\cH)$.
From Eq.~\eqref{eq:h_hp_bound},
we have $\disc(P,\Ph) \leq 4qM^{q-1}\fR_{n}(\cH) + M^q\sqrt{ \frac{ \ln (1/\delta) }{ 2n } }$.
Based on the triangle inequality, we have $\disc(P, Q) \leq \disc(P, \Ph) + \disc(Q, \Qh) + \disc(\Ph, \Qh)$.
For any $\delta>0$, the following inequality holds with probability at least $1-\delta$:
\begin{align*}
\disc(P, Q) &\leq \disc(\Ph, \Qh) + 8qM^{q-1}\fR_n(\cH) 
 + 2M^q \sqrt{ \frac{ \ln (2/\delta) }{ 2n } } .
\end{align*}
Applying the triangle inequality, we have,
\begin{align*}
&L_P(h, f_P) \leq L_P(h, h_Q^\ast) + L_P(h_Q^\ast, h_P^\ast) + L_P(h_P^\ast, f_P) \\
&\qquad \leq L_Q(h, h_Q^\ast) + \disc(P, Q)  + L_P(h_Q^\ast, h_P^\ast)
  + L_P(h_P^\ast, f_P) .
\end{align*}
We have, for any $\delta > 0$, the following inequality holds with probability at least $1-\delta$ for any $h \in \cH$:
\begin{align*}
L_P(h,f_P) &\leq \Lh_Q(h, h_Q^\ast) + L_P(h_Q^\ast, h_P^\ast) 
+ L_P(h_P^\ast, f_P) \\
&\quad + \disc(\Ph, \Qh) + 10qM^{q-1}\fR_n(\cH) + 3M^q\sqrt{ \frac{ \ln(3/\delta) }{ 2n } }.
\end{align*}
Replacing $\fR_n(\cH)$ with the upper bound, we obtain Proposition~\ref{thm:gen_err}.
\end{proof}

%\subsubsection{Proofs of approximability}
\subsection{Proofs of approximability}

Before going into the proof of the two results above,
we define PIPs as follows \cite{JCSS:Raghavan:1988}.

\begin{Definition}[Packing integer program]
Given  $m$ vectors $\bA_1, \ldots, \bA_m \in \Re^{n}_{\geq 0}$,
a capacity vector $\bB \in \Re^{m}_{\geq 0}$, and
a weight vector $\bw \in \Re^{n}_{\geq 0}$,
the \emph{packing integer program (PIP)} is defined as the following problem:
\begin{align*}
\max_{\ba \in \{0,1\}^n}\left\{ \bw^\top \ba \mid \bA_j^\top \ba \leq B_j, \forall j \in [m] \right\}.
\end{align*}
We define the column sparsity as
$S := \max_{i \in [n]} | \{ j \in [m] : A_{j,i} > 0 \} |$.
\end{Definition}

%PIPs include a variety of combinatorial optimization prolems such as  matching and knapsack.

\begin{proof}[Proof of Theorem~\ref{thm:pip-hard}]
Let $\bA_1, \ldots, \bA_m \in \Re^{n}_{\geq 0}$, $\bB \in \Re^{m}_{\geq 0}$, and $\bw \in \Re^{n}_{\geq 0}$
be an instance of PIP.
We can construct an instance of \texttt{NAO\_01} in polynomial time
such that the following conditions are satisfied:
\begin{itemize}
    \item $\bphi$ and $\cB$ satisfy that
    $\frac{1}{|\cB|} \sum_{\bx \in \cB} \bphi(\bx) = \bw$,
    \item the number of inequality constraints is $s = m$,
    the number of equality constraints is $t = 0$, and
    \item for each $j \in [s]$,
    it holds that $\bb_j = \bA_{j}$ and $d_j = B_j$.
\end{itemize}
It is easy to verify that
the resulting instance of \texttt{NAO\_01}
is exactly equivalent to a given instance of PIP;
the inapproximability result is thus obvious (see, \citet{TheoComp:Bansal+etal:2012,TheoComp:Zuckerman:2007}).
\end{proof}

\begin{proof}[Proof of Theorem~\ref{thm:pip-approx}]
Fix an \texttt{NAO\_mix} instance
$\bphi(\cdot), \cB, \{\bb_j\}_{j \in [s]}, \{d_j\}_{j \in [s]}, \{\bc_j\}_{j \in [t]}, \{e_j\}_{j \in [t]}$.

We first partition an $n$-dimensional domain-attribute vector $\ba$ into
continuous domain attributes and binary domain-attributes.
Let $\cI_{re}$ and $\cI_{bi}$ be the set of indices for continuous variables and binary variables, respectively.
Let us denote
$\ba_{re} = (a_i)_{i \in \cI_{re}}$ and $\ba_{bi} = (a_i)_{i \in \cI_{bi}}$,
and denote
$\bw = \frac{1}{|\cB|}\sum_{\bx \in \cB} \bphi(\bx)$,
$\bw_{re} = (w_i)_{i \in \cI_{re}}$, and $\bw_{bi} = (w_i)_{i \in \cI_{bi}}$;
note that $\bw^\top \ba = \bw_{re}^\top \ba_{re} + \bw_{bi}^\top \ba_{bi}$.
%We abuse the notation by writing $h(\ba_{re},\ba_{bi}) := h(\ba)$ for any function $h$ that takes $\ba$ as an argument.
The original \texttt{NAO\_mix} (denoted P1) can be rewritten as
\begin{align*}
    \max_{\substack{\ba_{re} \in \Re_{\geq 0}^{\cI_{re}} \\ \ba_{bi} \in \{0,1\}^{\cI_{bi}}}}
    \left\{ \bw^\top \ba \mid \bb_j^\top \ba \leq d_j, \forall j \in [s]; \; \bc_j^\top \ba = e_j, \forall j \in [t] \right\}.
\end{align*}

We now describe the approximation algorithm for P1.
We first consider the linear programming (LP) relaxation of P1 (denoted LP1)
that relaxes ``$\ba_{bi} \in \{0,1\}^{\cI_{bi}}$'' to
``$\ba_{bi} \in \Re_{\geq 0}^{\cI_{bi}}$''.
Since LP1 is an LP instance,
we can solve it exactly in polynomial time, e.g., by the ellipsoid method, and
denote its optimal solution by $\tilde{\ba} \in \Re_{\geq 0}^{n}$.
We then create a new instance of \texttt{NAO\_01} (denoted P2)
where entries of $\ba_{re}$ are \emph{fixed} to entries of $\tilde{\ba}_{re}$, and relax each equality constraint ``$ \bc_j^\top \ba = e_j $''
to ``$ \bc_j^\top \ba \leq e_j $'', the resulting \texttt{NAO\_01} instance (denoted P2') by which is an instance of PIP whose
column sparsity is at most $S$.
We thus can use \citet{TheoComp:Bansal+etal:2012}'s algorithm
to find an $8S$-factor approximate solution
$ \bar{\ba}_{bi} \in \{0,1\}^{\cI_{bi}} $ for P2'.
We can increase some of the entries of $\bar{\ba}_{bi}$ until
the equality constraints are satisfied,
which does not decrease the objective value.
Finally, we return a domain-attribute vector $\ba$ defined as follows:
\begin{align}
    a_i =
    \begin{cases}
    \tilde{a}_i & \text{if } i \in \cI_{re}, \\
    \bar{a}_i & \text{if } i \in \cI_{bi}.
    \end{cases}    
\end{align}

Since the feasibility of $\ba$ is obvious,
we analyze its approximation ratio.
Let $\ba^*$ be an optimal solution for P1.
Observe that $\bw^{\top} \tilde{\ba} \geq \bw^{\top} \ba^*$
as $\tilde{\ba}$ is an optimal solution for LP1.
We then show that $\bw_{bi}^{\top} \bar{\ba}_{bi} \geq \frac{1}{8S} \bw_{bi}^{\top} \tilde{\ba}_{bi}$.
Recall that \citet{TheoComp:Bansal+etal:2012}'s algorithm returns
a feasible solution for a PIP instance that has an objective value
at least $\frac{1}{8S}$ times the optimum value of \emph{its LP relaxation}.
If $\bar{\ba}'_{bi}$ is an optimal solution for the LP relaxation of P2',
it holds that
$ \bw_{bi}^\top \bar{\ba}'_{bi} \geq \bw_{bi}^\top \tilde{\ba}_{bi} $; hence,
we have that $ \bw_{bi}^\top \bar{\ba}_{bi} \geq \frac{1}{8S} \bw_{bi}^\top \bar{\ba}'_{bi} \geq \frac{1}{8S} \bw_{bi}^\top \tilde{\ba}_{bi} $.
Consequently,
\begin{align*}
    \bw^\top \ba & = \bw_{re}^\top \tilde{\ba}_{re} + \bw_{bi}^\top \bar{\ba}_{bi} 
    \geq \bw_{re}^\top \tilde{\ba}_{re}+  \frac{1}{8S} \bw_{bi}^\top \tilde{\ba}_{bi} \\
    & \geq \frac{1}{8S} (\bw_{re}^\top \tilde{\ba}_{re} + \bw_{bi}^\top \tilde{\ba}_{bi})
    \geq \frac{1}{8S} \bw^\top \ba^*.
\end{align*}
$\ba$ is an $8S$-factor approximate solution to \texttt{NAO\_mix},
which completes the proof.
\end{proof}

\end{document}